\documentclass[letterpaper, 10 pt, conference]{ieeeconf}  % Comment this line out if you need a4paper

\IEEEoverridecommandlockouts                              % This command is only needed if 
\usepackage{graphics} % for pdf, bitmapped graphics files
\usepackage{epsfig} % for postscript graphics files
\usepackage{mathptmx} % assumes new font selection scheme installed
\usepackage{times} % assumes new font selection scheme installed
\usepackage{amsmath} % assumes amsmath package installed
\usepackage{amssymb}  % assumes amsmath package installed
\usepackage{comment}
\usepackage{multirow}
\usepackage{float} 
\usepackage{placeins}
\usepackage{microtype}
\usepackage{svg}

\usepackage{xcolor}
\usepackage{cite}
\usepackage[normalem]{ulem}
\usepackage{kotex}
\usepackage[utf8]{inputenc}
\usepackage[T1]{fontenc}
\usepackage{csquotes}
\usepackage{mathtools}
\usepackage{comment}
\usepackage{pifont} % preamble에 추가

\usepackage{amsthm,thmtools} % bold

\newcommand\blue[1]{\textcolor{blue}{#1}}

\newtheorem{theorem}{Theorem}
\newtheorem{lemma}{Lemma}
\newtheorem*{condition}{Condition}

\title{\LARGE \bf
Geometric Backstepping Control of Omnidirectional Tiltrotors Incorporating Servo–Rotor Dynamics for Robustness against\\Sudden Disturbances
}

\author{Jaewoo Lee$^{1*}$, Dongjae Lee$^{2*}$, Jinwoo Lee$^{1}$, Hyungyu Lee$^{3}$, Yeonjoon Kim$^{1}$, and H.~Jin~Kim$^{1}$
\thanks{*The first two authors contributed equally to this work.}
\thanks{\raggedright $^{1}$Department of Aerospace Engineering, Seoul National University (SNU), Seoul 08826, South Korea {\tt\small \{jaewoolee930, jinwoolee0728, 0831joon, hjinkim\}@snu.ac.kr}}%
\thanks{$^{2}$Robotics Institute, Carnegie Mellon University, Pittsburgh, PA 15217, USA {\tt\small dongjae2@andrew.cmu.edu}}%
\thanks{\raggedright $^{3}$Department of Mechanical Science and Engineering, University of Illinois Urbana-Champaign, Champaign, IL 61801, USA {\tt\small hyungyu2@illinois.edu}}%
\thanks{This work was supported by the National Research Foundation of Korea(NRF) grant funded by the Korea government(MSIT)(RS-2024-00436984). Dongjae Lee was supported by Basic Science Research Program through the National Research Foundation of Korea(NRF) funded by the Ministry of Education(RS-2025-02634317).}
}

\begin{document}

\maketitle
\thispagestyle{empty}
\pagestyle{empty}

% %%%%%%%%%%%%%%%%%%%%%%%%%%%%%%%%%%%%%%%%%%%%%%%%%%%%%%%%%%%%%%%%%%%%%%%%%%%%%%%%
\begin{abstract}
This work presents a geometric backstepping controller for a variable-tilt omnidirectional multirotor that explicitly accounts for both servo and rotor dynamics. Considering actuator dynamics is essential for more effective and reliable operation, particularly during aggressive flight maneuvers or recovery from sudden disturbances. While prior studies have investigated actuator-aware control for conventional and fixed-tilt multirotors, these approaches rely on linear relationships between actuator input and wrench, which cannot capture the nonlinearities induced by variable tilt angles. In this work, we exploit the cascade structure between the rigid-body dynamics of the multirotor and its nonlinear actuator dynamics to design the proposed backstepping controller and establish exponential stability of the overall system. Furthermore, we reveal parametric uncertainty in the actuator model through experiments, and we demonstrate that the proposed controller remains robust against such uncertainty. The controller was compared against a baseline that does not account for actuator dynamics across three experimental scenarios: fast translational tracking, rapid rotational tracking, and recovery from sudden disturbance. The proposed method consistently achieved better tracking performance, and notably, while the baseline diverged and crashed during the fastest translational trajectory tracking and the recovery experiment, the proposed controller maintained stability and successfully completed the tasks, thereby demonstrating its effectiveness.
\end{abstract}

\section{INTRODUCTION}

Fully actuated multirotors have attracted growing attention as a powerful means to overcome the inherent underactuation of conventional multirotor platforms \cite{rashad2020fully}. Owing to their full actuation capability, these platforms can generate horizontal forces and achieve translational motion without tilting their roll and pitch angles. Such characteristics enable a variety of applications, including the transportation of payloads without attitude changes~\cite{park2024palletrone}, physical interaction with the environment~\cite{cuniato2022power}, and aerial manipulation tasks~\cite{he2025flying}.

Building on these advances, omnidirectional multirotors have been developed to further equip with the hovering capability in arbitrary orientations~\cite{brescianini2016design,kamel2018voliro,ryll2015novel}. Their capability to hover at arbitrary orientations, referred to as omnidirectionality in attitude, has attracted considerable attention for tasks such as contact-based inspection of curved surfaces~\cite{bodie2021active} and expanding the workspace of aerial manipulation systems~\cite{lee2025autonomous}. These capabilities highlight the growing potential of omnidirectional multirotors in complex real-world operations.

\begin{figure}
    \centering
    \includegraphics[width=1.0\linewidth]{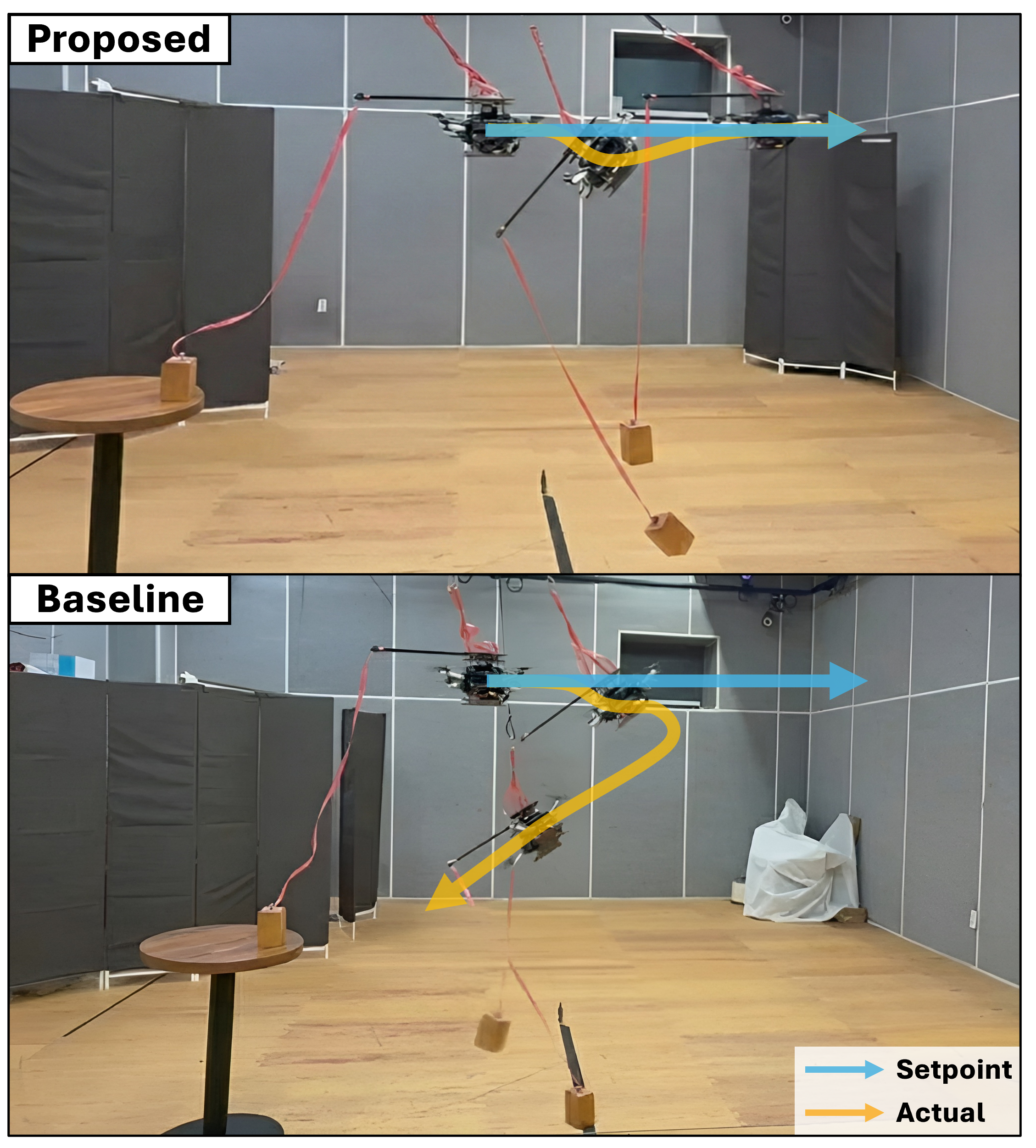}
    % \includegraphics[width=0.9\linewidth]{figures/thumbnail_v5.png}
    % \caption{Time-lapse composite images from the experiments, where blue arrows indicate the commanded setpoint direction. Panels \ding{172}~--~\ding{174} show the time sequence under a sudden rotational disturbance from a falling object connected to the multirotor with a red cable. The proposed controller allows the omnidirectional multirotor to follow the setpoint in order \ding{172}$ \rightarrow $\ding{173}$ \rightarrow $\ding{174} and remain stable, while the baseline diverges and crashes.}
    % \caption{Blue arrow shows the commanded setpoint direction. Panels ①–③ depict the time sequence as a sudden rotational disturbance occurs. The proposed controller follows the setpoint in order ①→②→③ and remains stable, whereas the baseline diverges and crashes.}
    % Yellow arrows indicate the commanded setpoint direction. Panels (a)–(c) show the time sequence as a sudden rotational disturbance from a falling object demands a rapid actuator response. The proposed controller, accounting for actuator dynamics, maintains stability and follows the setpoint, whereas the baseline controller cannot compensate and crashes.}
    \caption{Time-lapse composite images from the experiments. Blue arrows indicate the commanded setpoint direction, and yellow arrows trace the actual trajectory. A sudden disturbance is induced by a falling object connected to the omnidirectional multirotor via a red cable. The proposed controller (Top) successfully stabilizes the vehicle and follows the setpoint, whereas the baseline (Bottom) diverges and crashes.}
    
    \label{fig:thumbnail}
    \vspace{-15pt}
\end{figure}

To fully exploit the advantages of omnidirectional flight and to ensure reliable operation, it is essential to incorporate \textit{actuator dynamics}, the dynamics of rotors and servos, into the control design. Many prior studies, however, have relied on control allocation methods that compute the desired actuator input from the desired wrench and then directly treat this value as the actual actuator input~\cite{kamel2018voliro,lee2024autonomous,bodie2021active,lee2025autonomous}. Such an approach implicitly assumes that actuators can instantaneously realize any commanded input. This assumption may hold in slow or moderate maneuvers, but it becomes a critical issue when the control input must change rapidly, as in agile trajectory tracking or in recovering from sudden external disturbances. In these scenarios, ignoring actuator dynamics can degrade control performance or even lead to flight instability and crashes as shown in Fig. \ref{fig:thumbnail} baseline.

To address these challenges, this work proposes a control strategy for omnidirectional variable-tilt multirotors that explicitly accounts for both rotor and servo dynamics, in contrast to previous studies that ignored actuator dynamics or considered only a single component such as the servo or rotor~\cite{li2024servo,ryll2015novel,allenspach2020design,park2023design,seshasayanan2025robust}. We propose a geometric backstepping controller motivated by the cascaded structure between the rigid-body dynamics and the actuator dynamics. 

We first design a wrench controller for the rigid-body dynamics without considering the actuator dynamics and derive an actuator controller in a backstepping manner \cite{bang2009robust} that takes into account the gap between the desired wrench command and the actual wrench. Moreover, unlike prior studies that considered actuator dynamics only at the control allocation level and either locally linearized or entirely omitted their influence on the overall system dynamics \cite{allenspach2020design,cuniato2024allocation}, we conduct a Lyapunov-based stability analysis of the full closed-loop system that explicitly incorporates actuator dynamics. Through the analysis, we show that the entire system is exponentially stable. Finally, we further show that the proposed controller is also robust to parametric uncertainties in the time constant terms inherent to actuator dynamics.

The main contributions of this paper are summarized as follows:
\begin{itemize}
    \item A geometric backstepping control framework for variable-tilt omnidirectional multirotors that explicitly incorporates the first-order dynamics of both servos and motors.
    \item A rigorous stability analysis of the full closed-loop system, establishing exponential stability for known actuator time constants and ultimate boundedness under uncertainty in these parameters.
    \item Validation of the proposed controller through real-world experiments, including high-speed trajectory tracking and recovery from impulsive rotational disturbances. The results demonstrate that our method significantly outperforms a baseline controller, which fails to maintain stability during aggressive maneuvers.
\end{itemize}
%%%%%%%%%%%%%%%%%%%%%%%%%%%%%%

\section{RELATED WORK}

% To overcome the underactuation of quadrotors, platforms with tilted or tiltable rotors have been widely investigated. In particular, variable-tilt designs incorporate servomotors in their actuator. 
Many existing control strategies for conventional multirotors \cite{bouabdallah2005backstepping,lee2010geometric} and omnidirectional multirotors \cite{kamel2018voliro, lee2025autonomous, lee2024autonomous, bodie2021active} simplify the problem by assuming that actuators can respond instantaneously to commanded inputs. While effective for slow movements, this assumption fails during fast trajectory tracking or when recovering from external disturbances. %Especially in variable-tilt drones, in situations where rapid changes in thrust and tilt angle are required, the mismatch in response speed between the rotor and the servo can lead to performance degradation or instability \cite{park2023design}.

% To enhance the stability and tracking performance of multirotor vehicles, especially in dynamic scenarios that require rapid changes in control input, it is crucial to consider the dynamics of their actuators, namely rotors and servomotors. Many foundational control strategies for omnidirectional multirotors \cite{kamel2018voliro, lee2025autonomous, lee2024autonomous, bodie2021active} simplify the problem by assuming that actuators can respond instantaneously to commanded inputs. While effective for slow movements, this assumption fails during fast trajectory tracking or when recovering from external disturbances, situations where rapid changes in thrust and tilt angle are required, potentially leading to performance degradation or instability. \blue{일반적인 actuator dynamics를 고려하지 않았을 때 어찌되는지에 대한 general한 설명, 그 다음 특히 variable-tilt인 경우에는 추가적인 문제가 뭐가 있는지}

% To address these limitations, recent research has begun to incorporate actuator dynamics directly into the controller design. 
To address this issue, recent works begin to incorporate actuator dynamics directly into the controller design. For conventional multirotors, incorporating first-order rotor dynamics improves high speed tracking and robustness \cite{faessler2016thrust,tal2020accurate}. For fixed-tilt platforms where servo dynamics are absent, research has focused on the impact of rotor dynamics. Some studies have modeled the rotor angular speed dynamics as a first-order system and implemented an angular velocity feedback controller \cite{Brescianini2018}. Similarly, a geometric tracking controller was proposed that incorporated a first-order wrench dynamics model \cite{lee2025geometric}. However, for variable-tilt platforms, the wrench dynamics cannot be represented by a single first-order system due to the nonlinearity of the control allocation map. Consequently, variable-tilt multirotors require control strategies that explicitly accommodate these additional nonlinearities.%, successfully achieving almost global exponential stability without requiring direct measurements of the rotor state 

% For fixed-tilt omnidirectional platforms where servo dynamics are absent, research has focused on the impact of rotor dynamics. Some studies have modeled the rotor angular speed dynamics as a first-order system and implemented an angular velocity feedback controller \cite{Brescianini2018}. Similarly, a geometric tracking controller was proposed that incorporated a first-order wrench dynamics model, successfully achieving almost global exponential stability without requiring direct measurements of the rotor state \cite{lee2025geometric}. 

% \blue{In fixed-tilt multirotors, only rotor dynamics are present and the allocation mapping is linear, so the wrench dynamics can be modeled directly. In contrast, variable-tilt platforms exhibit two distinct time delay constants (rotor and servo), and the allocation depends nonlinearly on the servo angles, making wrench dynamics modeling substantially more difficult. Since prior work on fixed-tilt platforms exploits this linearity to design controllers, control on variable-tilt platforms requires methods that explicitly handle the additional nonlinearities. Also} 

Another challenge for variable-tilt platforms is the slower response of the tilt servos. To address this problem, a Smith predictor was employed to compensate for known servo delays \cite{ryll2015novel}. A quasi-decoupling controller was also developed, which uses the current servo angle for control allocation to achieve performance independent of the servo's response time \cite{Lee2021_carosq}. A similar approach to using the current servo angle in control allocation was also devised in \cite{park2023design} from the observation of slower servomotor response than the rotor response. A common limitation of these approaches is their focus on either rotor or servo dynamics in isolation. %\blue{fixed-tilt 멀티로터에서는 됐는데 variable-tilt에서는 뭐가 다른지 -- servo사용으로 인해 nonlinearity가 발생하고, 기존의 fixed-tilt 연구들은 linearity를 활용해서 제어기를 구성했기 때문에 이런 추가적인 nonlinearity를 다루는 법이 필요하다는 내용}

More recently, efforts have been made to address both actuator dynamics simultaneously. For instance, an approach using Nonlinear Model Predictive Control (NMPC) modeled both rotor thrust dynamics and servo angle dynamics as first-order systems \cite{li2024servo}. However, the substantial computational burden of NMPC remains particularly challenging in scenarios that require high-frequency control and rapid, large changes in actuator inputs. A different strategy, known as differential allocation, addresses actuator velocity limits and dynamics directly within the control allocation module \cite{allenspach2020design}. This approach was later extended to incorporate the power dynamics of the actuators, providing a more comprehensive model at the allocation level \cite{cuniato2024allocation}. However, while these methods are effective, their focus remains on solving the allocation problem itself. Consequently, they do not provide a formal stability analysis for the entire closed-loop system, where the vehicle's rigid-body dynamics are fully coupled with the actuator dynamics, and stability claims in these works are confined to simplified or linearized models. Furthermore, they rely on the assumption of a perfectly known model and do not provide a robustness analysis against model uncertainties.

\section{CONTROLLER DESIGN}

\subsection{Notations}

In this section, we define the key notations used throughout this work. For vectors $a,b \in \mathbb{R}^3$, the hat map is defined as $\hat{a} b = a \times b$, which maps a vector to a skew-symmetric matrix. The symbol $(\cdot)^{\vee}$ denotes its inverse transformation. $\mathrm{tr}(\cdot)$ denotes the trace of a matrix, 
and $\mathrm{sat}_{\sigma}(y)$ is defined as $\sigma \cdot \text{sign}(y)$ if $|y| > \sigma$, and $y$ otherwise.

%%%%%%%%%%%%% edited by 재우 %%%%%%%%%%%%%%%
% and $\mathrm{sat}_{\sigma}(y)$ is introduced as follows: 
% \begin{equation}
%     \mathrm{sat}_{\sigma}(y) = \begin{cases}
%     \sigma \cdot \text{sign}(y) & \text{if } |y| > \sigma, \\
%     y & \text{otherwise.}
%     \end{cases}
% \end{equation}
%%%%%%%%%%%%%%%%%%%%%%%%%%%%%%%%%%%%%%%%%%%
$\| \cdot \|$ represents the 2-norm when applied to a vector and the induced 2-norm when applied to a matrix. $\lVert a \rVert_\infty$ denotes the infinity norm of a vector $a$. The notation $\|\cdot\|_F$ represents the Frobenius norm of a matrix, which is defined as:
\begin{equation}
    \|A\|_F = \sqrt{\sum_{i,j} |A_{ij}|^2} \,\,.
\end{equation}
For a square matrix \( A \), the maximum and minimum eigenvalues are denoted by \( \lambda_{\max}(A) \) and \( \lambda_{\min}(A) \), respectively. The notation \( B^{\dagger} \) denotes the pseudo-inverse of a non-square matrix \( B \). 
The notation $\text{blockdiag}(A_1, A_2, \dots, A_k)$ denotes a block-diagonal matrix with blocks $A_1, A_2, \dots, A_k$ on the diagonal. $[a; b]$ represents the concatenation of two column vectors $a$ and $b$, which is defined as $[a; b] := [a^\top , b^\top]^\top$. The $n \times n$ identity matrix is denoted as $I_n$.

\subsection{System Dynamics}
We consider the variable-tilt multirotor as in \cite{kamel2018voliro,ryll2015novel,lee2024autonomous}. The total thrust $f$ and torque $\tau$ are expressed as a combination of rotor thrusts and servo angles: 
\begin{equation}
\begin{aligned} \label{eq: force and torque}
f &= \sum_{i=1}^{n} \prescript{B}{}{R}_{i} f_i \\
\tau &= \sum_{i=1}^{n} [l_i \times \prescript{B}{}{R}_{i} f_i ~~\pm~~ k_f\prescript{B}{}{R}_{i} f_i] 
\end{aligned}    
\end{equation}
where $n > 3$ is the number of rotors, $\prescript{B}{}{R}_{i} \in SO(3)$ is the corresponding rotation matrix of the $i^{th}$ servo angle expressed in the multirotor body frame, $f_i \in \mathbb{R}$ is the $i^{th}$ rotor thrust, and $l_i \in \mathbb{R}^3$ and $k_f \in \mathbb{R}$ are the displacement vectors of the $i^{th}$ rotor from the geometric center of the multirotor and the aerodynamic drag coefficient of the rotors. Using $f$ and $\tau$ defined in (\ref{eq: force and torque}), the equations of motion of the omnidirectional tiltrotor, including the dynamics of the rotors and servo motors, are given as follows:
% \begin{equation} \label{eq: system dynamics - rigid body}
% \begin{aligned}
% \ddot{p} &= \frac{1}{m} Rf - g e_3 + \Delta_p \\
% \dot{R}  &= R \hat{\omega} \\
% \dot{\omega} &= J^{-1} (-\omega \times J\omega + \tau) + \Delta_R
% \end{aligned}
% \end{equation}
% \begin{equation} \label{eq: system dynamics - actuator}
% \begin{aligned}
% \dot{f_i} &= \frac{1}{\alpha_f}(f_{c_i} - f_i) \\
% \dot{\theta_i} &= \frac{1}{\alpha_\theta}(\theta_{c_i} - \theta_i)
% \end{aligned}
% \end{equation}
\begin{subequations}\label{eq:system dynamics}
\begin{align}
&\begin{aligned}
\ddot{p}   & = \frac{1}{m} Rf - g e_3 + \Delta_p \\
\dot{R}    & = R \hat{\omega} \\
\dot{\omega} & = J^{-1}\!\left(-\omega \times J\omega + \tau\right) + \Delta_R
\end{aligned}
\label{eq: system dynamics - rigid body} \\
&\begin{aligned}
\dot{f_i}  & = \frac{1}{\alpha_f}\bigl(f_{c_i} - f_i\bigr) \\
\dot{\theta_i} & = \frac{1}{\alpha_\theta}\bigl(\theta_{c_i} - \theta_i\bigr)
\end{aligned}
\label{eq: system dynamics - actuator}
\end{align}
\end{subequations}
where $p,\omega \in \mathbb{R}^3$ and $R \in SO(3)$ are the position, body angular velocity, and the orientation of the multirotor. $i^{th}$ rotor thrust and servomotor angle are $f_i$ and $\theta_i$, respectively. $m,g\in \mathbb{R}$ are mass of the multirotor and the gravitational acceleration constant, $J \in \mathbb{R}^{3 \times 3}$ is the mass moment of inertia of the multirotor, and $\alpha_f, \alpha_\theta \in \mathbb{R}$ are the actuator time constants. 
%%%%%%%%% edited by 재우 %%%%%%%%%%%
We model the thrust directly as a first-order system to bypass the nonlinearity of the rotor speed-to-thrust mapping, a choice that balances modeling accuracy with controller simplicity, as discussed in \cite{lee2025geometric}.
%%%%%%%%%%%%%%%%%%%%%%%%%%%%%%%%%%%
$\Delta_p, \Delta_R \in \mathbb{R}^3$ are constant translational and rotational disturbance. Lastly, $f_{ci}, \theta_{ci} \in \mathbb{R}$ are the rotor thrust and servo angle commands.

For the ease of controller design, we define the vector $u \in \mathbb{R}^{2n}$ as
\begin{align*}
u &= [f_1 \cos\theta_1 , f_1\sin \theta_1, f_2 \cos \theta_2, f_2 \sin \theta_2 ,\ldots,f_n \sin \theta_n]^\top. 
\end{align*}
Accordingly, the resulting total wrench is computed as
\begin{align} \label{eq: mu}
\mu := \begin{bmatrix} f \\ \tau \end{bmatrix} = Bu
\end{align}
where $B \in \mathbb{R}^{6 \times 2n}$ denotes the allocation matrix \cite{kamel2018voliro}. Then, the rigid-body dynamics (\ref{eq: system dynamics - rigid body}) can be reformulated as
\begin{equation} \label{eq: system dynamics - rigid body simplified}
    \begin{aligned}
        \begin{bmatrix}
            \ddot{p} \\ \dot{\omega}
        \end{bmatrix} = \begin{bmatrix}
            -g e_3 \\ J^{-1} (- \omega \times J \omega)
        \end{bmatrix} + 
        \begin{bmatrix}
            \frac{1}{m}R & 0 \\
            0 & J^{-1}
        \end{bmatrix} 
        \mu + \begin{bmatrix}
            \Delta_p \\ \Delta_R
        \end{bmatrix}.
    \end{aligned}
\end{equation}
The actual control input, denoted by \(u_c\), is defined as follows:
\begin{align*}
u_c &= [f_{c_1}, \theta_{c_1},f_{c_2}, \theta_{c_2} \ldots, f_{c_n},\theta_{c_n} ]^\top 
\end{align*}
To characterize $\dot{u}$, we examine the time derivatives of its first two components:
\begin{align*}
\begin{bmatrix}
    \dot{u_1} \\ \dot{u_2}
\end{bmatrix} 
&= 
\begin{bmatrix}
    \dot{f_1}\cos\theta_1 - f_1\sin\theta_1\dot{\theta_1} \\ \dot{f_1}\sin\theta_1 - f_1\cos\theta_1\dot{\theta_1}
\end{bmatrix} \\
% &=
% \begin{bmatrix}
%     -\frac{f_1\cos\theta_1}{\alpha_f}+\frac{f_1\sin\theta_1}{\alpha_\theta}\theta_1 \\ 
%     -\frac{f_1\sin\theta_1}{\alpha_f}-\frac{f_1\cos\theta_1}{\alpha_\theta}\theta_1
% \end{bmatrix}
% +
% \begin{bmatrix}
%     \frac{\cos\theta_1}{\alpha_f} & -\frac{f_1\sin\theta_1}{\alpha_\theta} \\ 
%     \frac{\sin\theta_1}{\alpha_f} & \frac{f_1\cos\theta_1}{\alpha_\theta}
% \end{bmatrix}
% \begin{bmatrix}
%     f_{c_1} \\ \theta_{c_1}
% \end{bmatrix} \\
&=
\begin{bmatrix}
    -\frac{u_1}{\alpha_f}+\frac{u_2}{\alpha_\theta}\tan^{-1}\frac{u_2}{u_1} \\ 
     -\frac{u_2}{\alpha_f}-\frac{u_1}{\alpha_\theta}\tan^{-1}\frac{u_2}{u_1}
\end{bmatrix}
+
\begin{bmatrix}
    \frac{1}{\alpha_f}\frac{u_1}{\sqrt{u_1^2+u_2^2}} & -\frac{u_2}{\alpha_\theta} \\ 
    \frac{1}{\alpha_f}\frac{u_2}{\sqrt{u_1^2+u_2^2}} & \frac{u_1}{\alpha_\theta}
\end{bmatrix}
\begin{bmatrix}
    f_{c_1} \\ \theta_{c_1}
\end{bmatrix} \\
&=: \zeta_1(u) + \eta_1(u)
\begin{bmatrix}
    f_{c_1} \\ \theta_{c_1}
\end{bmatrix}
\end{align*}
Generalizing the above result to all columns, $\dot{u}$ can be written as follows:
\begin{align} \label{eq: system dynamics - actuator simplified}
\dot{u} = \zeta(u)+\eta(u)u_c
\end{align}
where $\zeta(u) \coloneqq [\zeta_1 ; \zeta_2;\dotsc;\zeta_n]\in\mathbb{R}^{2n}$, $\eta(u) \coloneqq \text{blockdiag}(\eta_1,\eta_2,\dotsc,\eta_n)\in\mathbb{R}^{2n\times 2n}$.
To avoid ill-posedness of (\ref{eq: system dynamics - actuator simplified}), we assume that $f_i \neq 0$ $\forall i$.
% \begin{align}
% \zeta(u) &\coloneqq [\zeta_1 ; \zeta_2;\dotsc;\zeta_n]\in\mathbb{R}^{2n},\notag\\
% \eta(u) &\coloneqq \text{blockdiag}(\eta_1,\eta_2,\dotsc,\eta_n)\in\mathbb{R}^{2n\times 2n}.\notag
% \end{align}

% $\zeta(u) = [\zeta_1 ; \zeta_2;\dotsc;\zeta_n]\in\mathbb{R}^{2n}$, $\eta(u)=\text{blockdiag}(\eta_1,\eta_2,\dotsc,\eta_n)\in\mathbb{R}^{2n \times 2n}$. 

\begin{comment}
{\color{blue} 이거 note로 설명하든지 아님 두번쨰 bullet은 assumption으로 추가
\begin{itemize}
  \item The ambiguity of \(\tan^{-1}\!\bigl(\tfrac{u_{i+1}}{u_i}\bigr)\) when \(u_i = 0\) is resolved by employing the \(\operatorname{atan2}(u_{i+1},\,u_i)\) function.
  \item Since the condition \(\sqrt{u_i^2 + u_{i+1}^2} = f_i = 0\) is infeasible,$f_i$ should not be zero.
  \item $det(\eta(u)) = \frac{f_1f_2f_3f_4}{\alpha_f^4\alpha_\theta^4}$ 
\end{itemize}
}
\end{comment}

\subsection{Backstepping Controller Design}

Observing the structures of (\ref{eq: system dynamics - rigid body}) and (\ref{eq: system dynamics - actuator}), once $f_{ci}$ and $\theta_{ci}$ are determined, $f$ and $\theta$ are computed from the actuator dynamics, which then define the rigid-body dynamics. This forms a cascaded structure. Inspired by \cite{bang2009robust}, we exploit this cascaded structural property to design a backstepping controller and ensure the stability of the overall system. The proposed backstepping controller is organized as follows: first, a nominal controller for the desired control input $\mu_d$ is constructed by considering only the rigid-body dynamics while neglecting the actuator dynamics. Next, using the difference between the desired control input $\mu_d$ and the actual control input $\mu$, defined as $e_\mu = \mu - \mu_d$, a candidate Lyapunov function is formulated. By analyzing this function, we derive the final actuator-level control input $u_c$ that guarantees system stability.

Let us first revisit $B u = \mu = \mu_d + e_\mu$ where $e_\mu = \mu - \mu_d$, and let $\mu_d = [\mu_{d,1};\mu_{d,2}]$, $e_\mu = [e_{\mu,1};e_{\mu,2}]$ where $\mu_1, \mu_2, e_{\mu,1}, e_{\mu,2} \in \mathbb{R}^3$. 
\begin{comment}
To apply the backstepping control method, we treat $Bu$ as a virtual control input. It can be expressed using the auxiliary variables 
$\alpha = [\alpha_1;\alpha_2]$, $s=[s_1 ; s_2]$ where $\alpha_1, \alpha_2, s_1, s_2\in\mathbb{R}^3$. \blue{$\alpha$랑 $s$랑이 각각 어떤 의미인지에 대한 설명. $\alpha$는 nominal control input (error를 0으로 보내기), $s$는 actuator dynamics를 고려해서 그 차이를 0으로 보내는 input}
\begin{align}
Bu &= \begin{bmatrix} \alpha_1 \\ \alpha_2\end{bmatrix} + \begin{bmatrix} s_1 \\ s_2 \end{bmatrix}
\end{align}
\end{comment}
Errors in position $e_p$, linear velocity $e_v$, rotation $e_R$, and angular velocity $e_\omega$ are defined as follows:
\begin{equation}
\begin{aligned}
e_p &= p - p_d, & e_v &= v - v_d, \\
e_R &= \frac{1}{2} (R_d^\top R-R^\top R_d)^\vee, & e_\omega &= \omega-R^\top  R_d \omega_d 
\end{aligned}
\end{equation}
Then, the error dynamics for the rigid-body dynamics (\ref{eq: system dynamics - rigid body simplified}) can be computed as follows:
\begin{equation} \label{eq: error dynamics - rigid body dynamics}
    \begin{bmatrix}
        \dot{e}_v \\ \dot{e}_\omega
    \end{bmatrix} = F + G (\mu_d + e_\mu) + \Delta
\end{equation}
where 
\begin{equation*}
\begin{gathered}
    F = \begin{bmatrix}
            -g e_3 - \dot{v}_d \\ J^{-1} (- \omega \times J \omega) + \hat{\omega}R^\top R_d \omega_d - R^\top R_d\dot{\omega}_d
        \end{bmatrix}, \\
    G = \begin{bmatrix}
            \frac{1}{m}R & 0 \\ 0 & J^{-1}
        \end{bmatrix}, \quad \Delta = \begin{bmatrix}
            \Delta_p \\ \Delta_R
        \end{bmatrix}.
\end{gathered}
\end{equation*}
\begin{comment}
\begin{equation} \label{eq: error dynamics open loop}
\begin{aligned}
\dot{e}_p &= e_v \\ 
\dot{e}_v &= \tfrac{1}{m} Rf - g e_3 - \dot{v_d} + \Delta_p \\
\dot{e}_R &= \tfrac{1}{2}(tr(R^\top R_d)I-R^\top R_d) =: C(R_d^\top R)e_\omega \\
J\dot{e}_\omega %&= J\dot{\omega} + J(\hat{\omega}R^\top R_d \omega_d - R^\top R_d\dot{\omega_d}) \\
&= -\omega\times J\omega + \tau + J(\hat{\omega}R^\top R_d \omega_d - R^\top R_d\dot{\omega_d}) + \Delta_R
\end{aligned}
\end{equation}
\end{comment}
Define $\mu_d$ using \cite{goodarzi2013geometric_arxiv} as the following, which guarantees exponential stability if $e_\mu = 0$:
\begin{equation} \label{eq: nominal control input}
\begin{aligned}
\mu_{d,1} &= mR^\top (-k_{tp}e_p-k_{td}e_v - k_{ti} \text{sat}_{\sigma_1}(e_{pi})+ ge_3 + \dot{v}_d) \\
\mu_{d,2} &= \omega \times J\omega - J(\hat{\omega}R^\top R_d \omega_d - R^\top R_d\dot{\omega}_d) - k_{rp}e_R \\ 
&~~~- k_{rd}e_\omega - k_{ri} \text{sat}_{\sigma_2}(e_{ri})
\end{aligned}
\end{equation}
% \begin{align}
% f &= mR^\top (-k_{tp}e_p-k_{td}e_v - k_{ti}e_{ti} +ge_3 + \dot{v_d}) \\
% \tau &= \omega \times J\omega - J(\hat{\omega}R^\top R_d \omega_d - R^\top R_d\dot{\omega_d}) - k_{rp}e_R - k_{rd}e_\omega - k_{ri}e_{ri}
% \end{align}
where, $k_{tp}, k_{td}, k_{rp}, k_{rd},  k_{ti},  k_{ri}$ are positive constants. For positive constants $c_1$ and $c_2$, $e_{pi}$ and $e_{ri}$ are defined as
\begin{equation} \label{eq: integral states}
    e_{pi} = \int_0^t e_v(\tau) + c_1 e_p(\tau) d\tau, \quad e_{ri} = \int_0^t e_\omega(\tau) + c_2 e_R(\tau) d\tau.
\end{equation}

Next, consider the candidate Lyapunov functions for the translational and rotational dynamics as
\begin{align}
V_1 &:= \tfrac{1}{2}k_{tp}\|e_p\|^2 \;+\; \tfrac{1}{2}\|e_v\|^2 \;+\; c_1\,e_p^\top e_v  \notag \\ &~~~ + \int_{\frac{\Delta_p}{k_{pi}}}^{e_{pi}} (k_{pi} \, \text{sat}_{\sigma_1}(\gamma) - \Delta_p) \, \cdot  d\gamma , \label{eq:V1_def}\\
V_2 &:= \tfrac{1}{2}e_\omega^\top J e_\omega \;+\; k_{rp}\,\Psi(R,R_d) \;+\; c_2\,e_R^\top e_\omega \notag \\ &~~~ + \int_{\frac{\Delta_R}{k_{ri}}}^{e_{ri}} (k_{ri} \,\text{sat}_{\sigma_2}(\gamma) - \Delta_R) \, \cdot d\gamma , \label{eq:V2_def}
\end{align}
where $\Psi(R,R_d)=\tfrac{1}{2}\operatorname{tr}(I-R_d^\top R)$. Positive-definiteness of the two candidate functions can be easily satisfied by defining $k_{pi} \sigma_1 > \lVert \Delta_p \rVert_\infty$ and $k_{ri} \sigma_2 > \lVert \Delta_R \rVert_\infty$ \cite{goodarzi2013geometric_arxiv}. Unlike \cite{goodarzi2013geometric_arxiv} where no actuator dynamics is considered, there exists the difference $e_\mu$ in the system dynamics (\ref{eq: error dynamics - rigid body dynamics}) which interrupts exponential stability guarantee. To simultaneously account for the influence of $e_\mu$ on the system dynamics, we propose an augmented candidate Lyapunov function $V$ as 
\begin{equation}
    V = \frac{1}{2} e_\mu^\top e_\mu + V_1 + V_2
\end{equation}
whose derivative is then used to obtain the final actuator control input $u_c$. We define $u_c$ as follows:
\begin{equation} \label{eq: uc}
u_c = \eta^{-1}B^\dagger(\dot{\mu}_d-B\zeta-k_\mu e_\mu-\kappa)
\end{equation}
% \begin{equation} \label{eq: uc}
% \begin{aligned}
% u_c &= \eta^{-1}B^\dagger(\dot{\mu}_d-B\zeta-k_\mu e_\mu-\sigma)\\
% &= \eta^{-1}B^\dagger(\dot{\mu}_d - B\zeta + k_\mu\mu_d -k_\mu \mu - \sigma)
% \end{aligned}
% \end{equation}
where $k_\mu > 0$ is a control parameter and 
\begin{align} \label{eq: kappa}
    \kappa = \begin{bmatrix}
        \frac{c_1}{m}R^\top e_p + \frac{1}{m}R^\top e_v \\ c_2J^{-1}e_R+e_\omega
    \end{bmatrix}.
\end{align}
$\zeta, \eta$ can be found from (\ref{eq: system dynamics - actuator simplified}) and $B$ appears in (\ref{eq: mu}). 

\section{STABILITY \& ROBUSTNESS ANALYSIS}

In the preceding section, a controller was designed using the backstepping approach. This section establishes exponential stability when the actuator time constant is known, and then demonstrates ultimate boundedness when the constant is unknown but its bound is available. Before these analyses, the conditions on the control gains and a lemma commonly used in both cases are presented.

\begin{condition}
For the constant parameters $c_1, c_2$ appearing in (\ref{eq: integral states}), (\ref{eq:V1_def}), (\ref{eq:V2_def}) and (\ref{eq: kappa}) and $k_{pi}, k_{ri}, \sigma_1, \sigma_2$ appearing in (\ref{eq:V1_def}) and (\ref{eq:V2_def}), we impose the following conditions:
\begin{equation} \label{eq: conditions}
\begin{aligned}
    c_1 &< min(\sqrt{k_{tp}}~,~\frac{4k_{tp}k_{td}}{k_{td}^2 + 4k_{tp}}) \\
    c_2 &< min(\sqrt{k_{rp}\lambda_{min}(J)}~,~\frac{4\lambda_{min}^2(J)  k_{rp}k_{rd}}{4k_{rp}\lambda_{min}^2(J) + k_{rd}^2\lambda_{max}(J)})\\
    k_{pi} \sigma_1 &> \lVert \Delta_p \rVert_\infty, \quad k_{ri} \sigma_2 > \lVert \Delta_R \rVert_\infty.
\end{aligned}
\end{equation}
\end{condition}

\begin{lemma} \label{lem:Vbounds}
Assume that (\ref{eq: conditions}) holds. Then, for $z_1 = [\|e_p\|;\|e_v\|]$ and $z_2 = [\|e_R\|;\|e_\omega\|]$, the candidate Lyapunov function $V$ is bounded by the following:
\begin{equation} \label{eq: V bound}
\begin{gathered}
    \tfrac{1}{2} e_\mu ^\top e_\mu + z_1^\top M_{11} z_1 + z_2^\top M_{21} z_2 + V_I
    \;\le\; V
    \;\le\; \\
    \tfrac{1}{2} e_\mu ^\top e_\mu + z_1^\top M_{12} z_1 + z_2^\top M_{22} z_2 + V_I,
\end{gathered}
\end{equation}
where $M_{11}, M_{12}, M_{21}, M_{22}$ are positive definite matrices and $V_I$ is positive definite.
\end{lemma}
\begin{proof}
From \cite{goodarzi2013geometric_arxiv}, $V_1$ and $V_2$ are bounded by 
\begin{gather*}
z_1^\top  M_{11} z_1 + V_{I_p}   \le V_1 \le z_1^\top  M_{12} z_1 + V_{I_p} \\
z_2^\top  M_{21} z_2 + V_{I_R} \le V_2 \le z_2^\top  M_{22} z_2 + V_{I_R}
\end{gather*}
where the matrices $M_{11}$, $M_{12}, M_{21}, M_{22}$ are given by
\begin{gather*}
M_{11} = \frac{1}{2} \begin{bmatrix} k_{tp} & -c_1 \\ -c_1 & 1 \end{bmatrix},~~
M_{12} = \frac{1}{2} \begin{bmatrix} k_{tp} & c_1 \\ c_1 & 1 \end{bmatrix} \\
M_{21} = \frac{1}{2} \begin{bmatrix} k_{rp} & -c_2 \\ -c_2 & \lambda_{min}(J) \end{bmatrix},~~
M_{22} = \frac{1}{2} \begin{bmatrix} \frac{2k_{rp}}{2-\psi_2} & c_2 \\ c_2 & \lambda_{max}(J) \end{bmatrix},
\end{gather*}
and $V_{I_p}$ and $V_{I_R}$ are defined as 
\begin{equation*}
\begin{aligned}
    V_{I_p} &= \int_{\frac{\Delta_p}{k_{pi}}}^{e_{pi}} (k_{pi} \, \text{sat}_{\sigma_1}(\gamma) - \Delta_p) \, \cdot  d\gamma \\
    V_{I_R} &= \int_{\frac{\Delta_R}{k_{ri}}}^{e_{ri}} (k_{ri} \,\text{sat}_{\sigma_2}(\gamma) - \Delta_R) \, \cdot d\gamma.
\end{aligned}
\end{equation*}
From (\ref{eq: conditions}), $M_{11}, M_{12}, M_{21}, M_{22}$ are all positive-definite \cite{goodarzi2013geometric_arxiv} and $V_{I_p} > 0$ and $V_{I_R} > 0$ unless $k_{pi} e_{pi} = \Delta_p$ and $k_{ri} e_{ri} = \Delta_R$. Defining  $V_I \coloneqq V_{I_p} + V_{I_R}$, this completes the proof.
\end{proof}

\subsection{Exponential stability analysis with known actuator dynamics}

\begin{theorem}
    Assume that conditions (\ref{eq: conditions}) hold and that parameters $\alpha_f$, $\alpha_\theta$ in (\ref{eq: system dynamics - actuator}) are known. Then, the closed-loop system composed of (\ref{eq: system dynamics - rigid body simplified}), (\ref{eq: system dynamics - actuator simplified}) and (\ref{eq: uc}) is exponentially stable.
\end{theorem}
\begin{proof}
From the definitions of $V_1$ and $V_2$ in (\ref{eq:V1_def}) and (\ref{eq:V2_def}), we compute the time derivatives of $V_1$ and $V_2$ using the system dynamics (\ref{eq: system dynamics - rigid body simplified}) and the desired control input (\ref{eq: nominal control input}) as follows:
\begin{equation} \label{eq: V1dot and V2dot}
\begin{aligned}
\dot{V}_1 &\le e_{\mu,1}^\top (\frac{c_1}{m}R^\top e_p + \frac{1}{m}R^\top e_v) - z_1^\top W_1z_1 \\
\dot{V}_2 & \le e_{\mu,2}^\top (c_2J^{-1}e_R+e_\omega) - z_2^\top W_2z_2
\end{aligned}
\end{equation}
where the matrices $W_1$ and $W_2$ are given by
\begin{align*}
W_{1} = \tfrac{1}{2} \begin{bmatrix} c_1k_{tp} & -\frac{1}{2}c_1k_{td} \\ -\frac{1}{2} c_1k_{td} & k_{td}-c_1 \end{bmatrix}, W_{2} = \tfrac{1}{2} \begin{bmatrix} \frac{c_2k_{rp}}{\lambda_{max}(J)} & -\frac{c_2k_{rd}}{2\lambda_{min}(J)} \\ -\frac{c_2k_{rd}}{2\lambda_{min}(J)} & k_{rd}-c_2 \end{bmatrix},
\end{align*}
% \begin{align*}
% W_{1} &= \frac{1}{2} \begin{bmatrix} c_1k_{tp} & -\frac{1}{2}c_1k_{td} \\ -\frac{1}{2} c_1k_{td} & k_{td}-c_1 \end{bmatrix}, \\
% W_{2} &= \frac{1}{2} \begin{bmatrix} \frac{c_2k_{rp}}{\lambda_{max}(J)} & -\frac{c_2k_{rd}}{2\lambda_{min}(J)} \\ -\frac{c_2k_{rd}}{2\lambda_{min}(J)} & k_{rd}-c_2 \end{bmatrix},
% \end{align*}
and $W_1, W_2$ are positive definite by (\ref{eq: conditions}) \cite{goodarzi2013geometric_arxiv}.

From the definition of $\kappa$ in (\ref{eq: kappa}), the definition of $e_\mu = \mu - \mu_d = Bu - \mu_d$, the actuator dynamics (\ref{eq: system dynamics - actuator simplified}), (\ref{eq: V1dot and V2dot}), and the actuator control input $u_c$ (\ref{eq: uc}),
\begin{equation}\label{eq:Vdot}
\begin{aligned}
    \dot{V} &\le e_{\mu}^\top (B\dot{u} - \dot{\mu}_d + \kappa) - z_1^\top W_1z_1 - z_2^\top W_2z_2\\
    &= e_{\mu}^\top (B(\zeta+\eta u_c)-\dot{\mu}_d + \kappa) - z_1^\top W_1z_1 - z_2^\top w_2z_2 \\
    &= -k_\mu e_\mu^\top e_\mu - z_1^\top W_1 z_1 - z_2^\top W_2 z_2
\end{aligned}
\end{equation}
where we used the fact that $B B^\dagger = I_6$. At the equilibrium point with $k_{pi} e_{pi} = \Delta_p$ and $k_{ri} e_{ri} = \Delta_R$, since the Lyapunov function $V$ is quadratically bounded by $[e_\mu;z_1;z_2]$ from Lemma \ref{lem:Vbounds} and its time-derivative is quadratically bounded by (\ref{eq:Vdot}), the system is exponentially stable \cite{goodarzi2013geometric_arxiv}.
\end{proof}

\subsection{Robustness analysis with uncertain actuator dynamics}

Real hardware deviates from an ideal first-order actuator model, so uncertainty in actuator dynamics must be taken into account. To substantiate this, we experimentally measured the step responses of the actuators. As shown in Fig. \ref{fig:rotor servo response}, the identified time constants vary with the command setpoint, and a single first-order fit leaves noticeable residuals~--~up to 30\% for the servo and up to 45\% for the rotor. We therefore model the actuator dynamics with bounded uncertainty and analyze the controller accordingly. Crucially, guaranteeing stability under such uncertainty removes the need to identify an exact time constant for every actuator.

\begin{figure}
    \centering
    \includegraphics[width=1.0\linewidth]{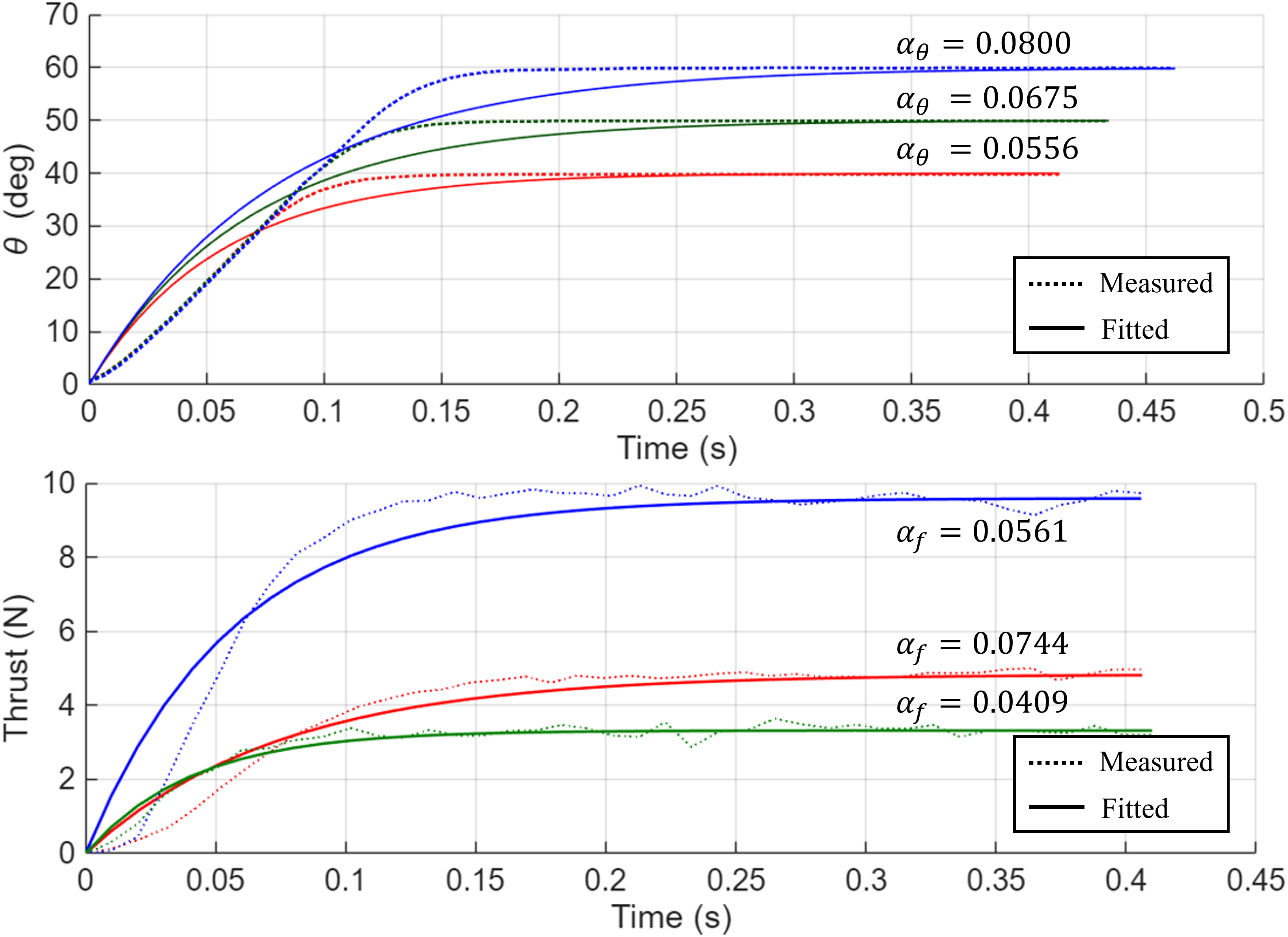}
    \caption{Step responses of the rotor and servo to step commands of varying amplitude. The actuator time constant varies with step size, highlighting the need for robustness to uncertainty in this parameter.}
    \label{fig:rotor servo response}
\end{figure}

In this subsection, we establish the boundedness of the state error under the proposed controller in the presence of uncertainties in both the tiltrotor rigid-body dynamics and the actuator dynamics. To model actuator uncertainty, we treat the rotor and servo time constants to be unknown but bounded deviations around nominal values:
\[
\alpha_f \in [\,\bar{\alpha}_f-\delta_f,\ \bar{\alpha}_f+\delta_f\,],\qquad
\alpha_\theta \in [\,\bar{\alpha}_\theta-\delta_\theta,\ \bar{\alpha}_\theta+\delta_\theta\,].
\]

% \cite{goodarzi2013geometric_arxiv} 

% We first define Lyapunov candidate function
% \begin{equation}
% V \;=\; \tfrac{1}{2}\,s^\top s \;+\; V_1 \;+\; V_2, \label{eq:V_def}
% \end{equation}
% where the translational and rotational terms are
% \begin{align}
% V_1 &:= \tfrac{1}{2}k_{tp}\|e_p\|^2 \;+\; \tfrac{1}{2}\|e_v\|^2 \;+\; c_1\,e_p^\top e_v  \notag \\ &~~~ + \int_{\frac{\Delta_p}{k_{pi}}}^{e_{pi}} (k_{pi} \, \text{sat}_{\sigma_1}(\gamma) - \Delta_p) \, \cdot  d\gamma , \label{eq:V1_def}\\
% V_2 &:= \tfrac{1}{2}e_\omega^\top J e_\omega \;+\; k_{rp}\,\Psi(R,R_d) \;+\; c_2\,e_R^\top e_\omega \notag \\ &~~~ + \int_{\frac{\Delta_R}{k_{ri}}}^{e_{ri}} (k_{ri} \,\text{sat}_{\sigma_2}(\gamma) - \Delta_R) \, \cdot d\gamma , \label{eq:V2_def}
% \end{align}
% with $\Psi(R,R_d)=\tfrac{1}{2}\operatorname{tr}(I-R_d^\top R)$.
% \\

\begin{lemma}\label{lem:zeta-eta-bound}
For $x = [z_1;z_2;e_\mu]$, assume that $x \in \Omega_c = \{x | V(x) \leq c \}$ for a positive constant $c$. Furthermore, assume that \( \| \dot{v}_d  \| \le a_v\), \( \| \omega_d \| \le a_\omega \), \( \| \dot{\omega}_{d} \| \le a_{\dot \omega} \), and $\alpha_f \in [\,\bar{\alpha}_f-\delta_f,\ \bar{\alpha}_f+\delta_f\,]$, $\alpha_\theta \in [\,\bar{\alpha}_\theta-\delta_\theta,\ \bar{\alpha}_\theta+\delta_\theta\,]$. Then, the norms of \( \Delta \zeta := \zeta - \bar{\zeta} \) and \( \Delta \eta := \eta - \bar{\eta} \) are bounded by a positive constant.
\end{lemma}

\begin{proof}
From (\ref{eq: V bound}) and the assumption that $V(x) \leq c$, 
\begin{equation} \label{eq: V-bound-c}
    \tfrac{1}{2} e_\mu ^\top e_\mu + z_1^\top M_{11} z_1 + z_2^\top M_{21} z_2 + V_I
    \;\le\; V
    \;\le\; c  \,\,\,.
\end{equation}
Since all terms on the left-hand side are positive, \( e_\mu, z_1, z_2 \) are bounded as follows:
\begin{equation} \label{eq: z-bound}
\|e_\mu\|\le \sqrt{2c},\,\,\,\,
\|z_1\|\le \sqrt{\frac{c}{\lambda_{\min}(M_{11})}},\,\,\,\,
\|z_2\|\le \sqrt{\frac{c}{\lambda_{\min}(M_{21})}} \,.
\end{equation}
Substituting this condition into (\ref{eq: nominal control input}), we get
% % --- alpha bounds (compact form) ---
\begin{equation} \label{eq: mu_d bound}
    \|\mu_d\| \;\le\; L_1\,\|z_1\| \;+\; L_2\,\|z_2\| \;+\; L_0
\end{equation}
where
\begin{equation} \label{eq: L0 L1 L2}
\begin{aligned} 
   &L_1 =  m\sqrt{k_{tp}^2+k_{td}^2} \\
   &L_2 = \sqrt{k_{rp}^2 + (\lambda _{max}(J)\sqrt{c \over \lambda_{min}(M_{21})}+
   3a_w\lambda_{max}(J)+k_{rd})^2} \\
   &L_0 = m(g+a_v+k_{ti}\sigma_1) + \lambda_{max}(J)(2a_w^2+a_{\dot w}) + k_{ri}\sigma_2 \, .
\end{aligned}
\end{equation}

From the equation \( Bu = \mu_d + e_\mu \), applying (\ref{eq: z-bound}) and (\ref{eq: mu_d bound}), we obtain:
\begin{equation}
\begin{aligned}
\|u\|
&\le \|B^\dagger\|\!\left\{
    \sqrt{2c}
    + L_1\,\sqrt{\frac{c}{\lambda_{\min}(M_{11})}}
    + L_2\,\sqrt{\frac{c}{\lambda_{\min}(M_{21})}}
    + L_0
\right\} \\
&\eqqcolon \rho(c).
\end{aligned}
\end{equation}
Since $\|u\|=\sqrt{\sum_{i=1}^n f_i^{\,2}} \le \rho(c) $, it follows that
\begin{equation} \label{eq:fmax}
f_{max}\,:=\,\max_{i} f_i\;\leq \,\rho(c).
\end{equation}

Using the initial assumption that \( \alpha_f \) and \( \alpha_\theta \) are bounded around their nominal values, we define \( d_f \) and \( d_\theta \) as follows:
\begin{align*}
\left|\frac{1}{\bar{\alpha}_f}-\frac{1}{\alpha_f}\right|
&\le \frac{\delta_f}{\bar{\alpha}_f(\bar{\alpha}_f-\delta_f)}
=: d_f, \\[2mm]
\left|\frac{1}{\bar{\alpha}_\theta}-\frac{1}{\alpha_\theta}\right|
&\le \frac{\delta_\theta}{\bar{\alpha}_\theta(\bar{\alpha}_\theta-\delta_\theta)}
=: d_\theta .
\end{align*}
Using \( d_f \) and \( d_\theta \), \( \Delta \zeta_i \) and \( \Delta \eta_i \) can be expressed as follows:

\begin{equation}\label{eq:delta zeta}
\Delta \zeta_i = f_i 
\begin{bmatrix}
    d_f \cos \theta_i \,-\, d_\theta \theta_i \sin \theta_i \\
    d_f \sin \theta_i \,+\, d_\theta \theta_i \cos \theta_i
\end{bmatrix},
\end{equation}

\begin{equation} \label{eq:delta eta}
\Delta \eta_i = 
\begin{bmatrix}
    d_f \cos \theta_i & -f_i \,d_\theta \sin \theta_i  \\
    d_f \sin \theta_i & f_i \,d_\theta \cos \theta_i 
\end{bmatrix}.
\end{equation}

By substituting \(f_{\max}\) derived in (\ref{eq:fmax}) into (\ref{eq:delta zeta}) and (\ref{eq:delta eta}), we can bound the perturbations as follows:
\begin{equation}\label{eta bound}
\|\Delta\eta\| \le \max\limits_{1 \le i \le n}\|\Delta\eta_i\|_F \le \sqrt{\,d_f^{\,2}+f_{\max}^{\,2}\,d_\theta^{\,2}\,},
\end{equation}
\begin{equation}\label{zeta bound}
\|\Delta\zeta\| = \sqrt{\sum\limits_{i=1}^n \|\Delta\zeta_i\|^2} \le \sqrt{n} \, f_{\max} \sqrt{\,d_f^{\,2} + \theta_{\max}^{\,2} \, d_\theta^{\,2}\,}.
\end{equation}

% By substituting \(f_{\max}\) derived in (\ref{eq:fmax}) into (\ref{eq:delta zeta}) and (\ref{eq:delta eta}), we can bound the perturbations as $\|\Delta\eta\| \le \max\limits_{1 \le i \le n}\|\Delta\eta_i\|_2 \le  \max\limits_{1 \le i \le n}\|\Delta\eta_i\|_F
% \le \sqrt{\,d_f^{\,2}+f_{\max}^{\,2}\,d_\theta^{\,2}\,}$, 
% $\|\Delta\zeta\| = \sqrt{\sum\limits_{i=1}^n \|\Delta\zeta_i\|^2} \le \sqrt{n} \, f_{\max} \sqrt{\,d_f^{\,2} + \theta_{\max}^{\,2} \, d_\theta^{\,2}\,}$.
% \begin{align}
% &\|\Delta\eta_i\| \le\|\Delta\eta_i\|_F
% \le \sqrt{\,d_f^{\,2}+f_{\max}^{\,2}\,d_\theta^{\,2}\,}, \\[2mm]
% &\|\Delta\zeta_i\| \le
% 2\,f_{\max}\,\sqrt{\,d_f^{2}+\theta_{\max}^{2}\,d_\theta^{2}\,}.
% \end{align}
\end{proof}

\begin{theorem}\label{thm:ISS stability}
Assume that $k_\mu > \tfrac{(1-\gamma)^2}{\gamma}\max\!\{\tfrac{M_1^{2}}{\lambda_{\min}(W_1)},\,\tfrac{M_2^{2}}{\lambda_{\min}(W_2)}\}$ and $f_i>0$ $\forall i$ with $\gamma = 1 - \|B\|\,\|\Delta\eta\,\eta^{-1}\|\,\|B^\dagger\| > 0$. Then, the error variable $x=[z_1;z_2;e_\mu]$ for the closed-loop system composed of (\ref{eq: system dynamics - rigid body simplified}), (\ref{eq: system dynamics - actuator simplified}) and (\ref{eq: uc}) with uncertain parameters in actuator time constant is bounded.

\begin{proof}
From \eqref{eq: uc}, the proposed control input is specified using nominal values of $\alpha_f, \alpha_\theta$ as
\begin{equation} \label{eq: uc nominal}
u_c = \bar\eta^{-1}B^\dagger(\dot{\mu}_d-B\bar\zeta-k_\mu e_\mu-\kappa)
\end{equation}
where $\bar{(\cdot)}$ denotes the nominal counterpart using $\bar{\alpha}_f, \bar{\alpha}_\theta$. Substituting the proposed control law \eqref{eq: uc nominal} into \eqref{eq:Vdot} yields
\begin{align} \label{eq:Vdot uncertainty}
\dot V \le
& -k_\mu\Big(1-\|B\|\,\|\Delta\eta\,\eta^{-1}\|\,\|B^{\dagger}\|\Big)\,\|e_\mu\|^{2}
  + \|e_\mu\|\,\|B\|\,\|\Delta\zeta\| \notag \\
& + \|B\|\,\|\Delta\eta\,\eta^{-1}\|\,\|B^{\dagger}\|\,\|\dot{\mu}_d-B\zeta-\kappa\| \notag \\
& - z_1^{\top}W_1 z_1 - z_2^{\top}W_2 z_2 .
\end{align}
% \begin{equation}\label{eq:Vdot uncertainty}
% \begin{aligned}
% \dot V \le
% & -k_\mu\Big(1-\|B\|\,\|\Delta\eta\,\eta^{-1}\|\,\|B^{\dagger}\|\Big)\,\|e_\mu\|^{2}
%   + \|e_\mu\|\,\|B\|\,\|\Delta\zeta\| \\
% & + \|B\|\,\|\Delta\eta\,\eta^{-1}\|\,\|B^{\dagger}\|\,\|\dot{\mu}_d-B\zeta-\kappa\| \\
% & - z_1^{\top}W_1 z_1 - z_2^{\top}W_2 z_2 .
% \end{aligned}
% \end{equation}
Assume that the initial state \(x\) satisfies \(V(x) \le c\) for some positive constant \(c\). 
By following a procedure similar to that used for (\ref{eq: mu_d bound}), 
it follows that there exist positive constants \(M_0\), \(M_1\), and \(M_2\) such that
\begin{equation}\label{eq: internal-term}
\begin{aligned}
\|\dot{\mu}_d-B\zeta-\kappa\|
&\le M_0 + M_1 \|z_1\| + M_2 \|z_2\|. \\[2pt]
\end{aligned}
\end{equation}

Substituting (\ref{eq: internal-term}) into \eqref{eq:Vdot uncertainty} and applying Young’s inequality, %\(ab \le \tfrac{\varepsilon}{2}\|a\|^2 + \tfrac{1}{2\varepsilon}\|b\|^2\), 
we obtain
\begin{equation}\label{eq:Vdot-bound}
\begin{aligned}
\dot V
&\le -\beta_1 \|e_\mu\|^2 - \beta_2 \|z_1\|^2 - \beta_3 \|z_2\|^2 + \beta_4
\end{aligned}
\end{equation}
where
\begin{equation*}
\begin{gathered}
\gamma  := 1- \|B\|\,\|\Delta\eta\,\eta^{-1}\|\,\|B^\dagger\|\\\
\beta_1 = \tfrac{1}{2}k_\mu\gamma, \beta_2 = \lambda_{m}(W_1) - \frac{(1-\gamma)^2\,{M_1}^{2}}{k_\mu\gamma}\\
\beta_3 = \lambda_{m}(W_2) - \frac{(1-\gamma)^2\,M_2^2}{k_\mu\gamma}, \beta_4 = \frac{\|B\|^2\,\|\Delta\zeta\|^2 + (1-\gamma)^2 M_0^2}{k_\mu \gamma}.
\end{gathered}
\end{equation*}
% \begin{equation}
% \begin{aligned}
% \gamma  :=& 1- \|B\|\,\|\Delta\eta\,\eta^{-1}\|\,\|B^\dagger\|\\\
% \beta_1 =& \tfrac{1}{2}k_\mu\gamma, \\
% \beta_2 =& \lambda_{m}(W_1) - \frac{(1-\gamma)^2\,{M_1}^{2}}{k_\mu\gamma}\\
% \beta_3 =& \lambda_{m}(W_2) - \frac{(1-\gamma)^2\,M_2^2}{k_\mu\gamma}\\
% \beta_4 =& \frac{\|B\|^2\,\|\Delta\zeta\|^2 + (1-\gamma)^2 M_0^2}{k_\mu \gamma}.\\
% \end{aligned}
% \end{equation}
Note that $\beta_2>0$ and $\beta_3>0$ by the assumption that $k_\mu \;>\; \tfrac{(1-\gamma)^2}{\gamma}\,\max\!\left\{ \tfrac{M_1^{2}}{\lambda_{\min}(W_1)} \;,\; \tfrac{M_2^{2}}{\lambda_{\min}(W_2)} \right\}$.

At the equilibrium point, according to Lemma \ref{lem:Vbounds}, $V$ is quadratically lower bounded by $x$. Thus, from (\ref{eq:Vdot-bound}), there exists a positive constant $c_v$ that satisfies $\dot{V} \leq -c_v V + \beta_4$, and this leads to boundedness of $\lVert x \rVert$ by comparison lemma~\cite{khalil2002nonlinear}.
\end{proof}

\end{theorem}

%%%%%%%%%%%%%%%%%%%%%%%%%%%%%%%%%%%%%%%%%%%%%%%%%%%%

%%%%%%%%%%%%%%%%%%%%%%%%%%%%%%%%%%%%%%%%%%%%%%%%%%%%%

%%%%%%%%%%%%%%%%%%%%%%%%%%%%%%%%%%%%%%%%%%%%%%%%%%%%%%%%%%%%%%%%%%%%%%%%%%%%%%%%%%%%%%%%%%

\section{RESULTS}

\subsection{Experimental setup}

We validate the proposed backstepping controller and evaluate its performance under uncertainty through hardware experiments. The tiltable quadrotor was equipped with Armattan 2450Kv rotors, APC $6\times4\times3.2$ propellers, and Dynamixel XC330-M288 servos. The flight controller utilized Pixhawk6C with PX4 software, and the on-board computer, ASUS NUC13, ran Robot Operating System 2 (ROS2) to directly control the actuators. For the odometry of the tiltrotor, we fused data from the indoor motion capture system, OptiTrack, and the IMU sensor data from the Pixhawk. The position and attitude controllers operated at 100 Hz and 200 Hz, respectively.

To generate the control inputs for the proposed backstepping controller, the total force and torque of the current state must be estimated. Additionally, to calculate the $\zeta$ matrix in (\ref{eq: system dynamics - actuator simplified}), we estimated the thrusts $f_i$ of each rotor and the state $\theta_i$ of the servos. Servomotors are typically equipped with encoders that provide direct angle feedback. The Dynamixel servos used in our system also include embedded encoders, which allow us to readily obtain estimates of the servo angles. In contrast, many ESCs for rotors do not report the rotor’s angular velocity. To address this limitation, we estimated the rotor thrust indirectly. Specifically, we obtained wrench estimates by solving the inverse rigid-body dynamics using linear and angular acceleration estimates provided by PX4, following the approach of \cite{lee2025geometric}.
% \blue{Servos usually offer state feedback; by contrast, ESCs with rotor-speed telemetry are rare, which makes this a critical design consideration. Servo motor angles were updated at 200 Hz using state feedback. The net force and torque were then estimated from sensor-measured linear and angular accelerations via inverse rigid-body dynamics.} 
To mitigate high-frequency noise in the acceleration data from the sensors, we applied a low-pass filter with a cutoff frequency of 20 Hz to the estimated force and torque. 
% The force of each rotor is estimated as follows:
% \begin{equation}
%     u_{\text{est}} = B^\dagger \mu_{\text{est}}, \quad f_i = \sqrt{u_{2i-1}^2 + u_{2i}^2} \quad \text{for } i = 1,2,3,4
% \end{equation}

For all subsequent real-world experiments, we compare the proposed backstepping controller against a baseline framework consisting of a geometric PID motion controller \cite{goodarzi2013geometric_arxiv} and a geometric allocation method \cite{kamel2018voliro} that does not account for actuator dynamics. The control gains were tuned during hovering at zero roll and pitch angles until sufficiently small translational and rotational errors were achieved. Under these conditions, the gains were further adjusted so that the baseline and the proposed method exhibited comparable performance levels.

To evaluate the performance of the proposed algorithm under aggressive maneuvers involving abrupt changes in control inputs, we designed three experimental scenarios. In the first and second experiments, we compared tracking performance during rapid translational and rotational motions, respectively. The third experiment involved recovering position and attitude under sudden external disturbance, which is also illustrated in Fig.~\ref{fig:thumbnail}. This disturbance was induced by suspending a 0.21 kg mass from a 2 kg drone with a string: the mass initially rested on a table, and as the drone moved laterally to the right, the mass applied an impulsive disturbance in the rotational direction.
% \blue{Leveraging the platform’s omnidirectional actuation—where translational and rotational motions are fully decoupled—we evaluate position tracking only in Experiment~1 and attitude tracking in Experiment~2.} In Experiment~3, to create a condition that requires instant changes in rotor and servo states, we inject an impulsive disturbance and assess whether the controller can recover the attitude while maintaining stable control.

\subsection{Experiment 1}

To evaluate agile position–tracking performance, we selected a lemniscate (figure-eight) trajectory. After takeoff, the vehicle maintained a constant altitude $z=1.2\,\mathrm{m}$ and tracked $x(t)=0.4\,\sin(\omega t), \ y(t)=0.3\,\sin(2\omega t)$,
% \begin{equation*}
%     x(t)=0.4\,\sin(\omega t), \qquad y(t)=0.3\,\sin(2\omega t),
% \end{equation*}
while holding a fixed zero attitude. The average speed over one period was varied by adjusting the angular frequency $\omega$ to achieve $0.8$, $1.0$, and $1.2\,\,\mathrm{m/s}$.

\begin{table}
\centering
\caption{Experiment~1 results: Position and orientation error.\\Better performance in bold.} %RMSE of position and orientation errors for the proposed method and the baseline across trials with different average speeds. The better-performing value between the proposed and the baseline is highlighted in bold.} 
\label{tab:table_lemni} 
\setlength{\tabcolsep}{10pt} 
\renewcommand{\arraystretch}{1.5}
\begin{tabular}{|c|l||c|c|c|}
\hline
\multicolumn{2}{|c||}{} & \multicolumn{3}{c|}{Average Velocity} \\ \cline{3-5} 
\multicolumn{2}{|c||}{} & 0.8 m/s & 1.0 m/s & 1.2 m/s \\ \hline \hline
\multirow{2}{*}{Proposed} & pos. [m]   & \textbf{0.050}  & \textbf{0.055}  & \textbf{0.074}  \\ \cline{2-5} 
                          & rot. [rad] & \textbf{0.057}  & \textbf{0.077}  & \textbf{0.132}  \\ \hline \hline
\multirow{2}{*}{Baseline} & pos. [m]   & 0.070  & 0.072  & $\times$      \\ \cline{2-5} 
                          & rot. [rad] & 0.072  & 0.085  & $\times$      \\ \hline
\end{tabular}
\end{table}

\begin{figure}
    \centering
    \includegraphics[width=1.0\linewidth]{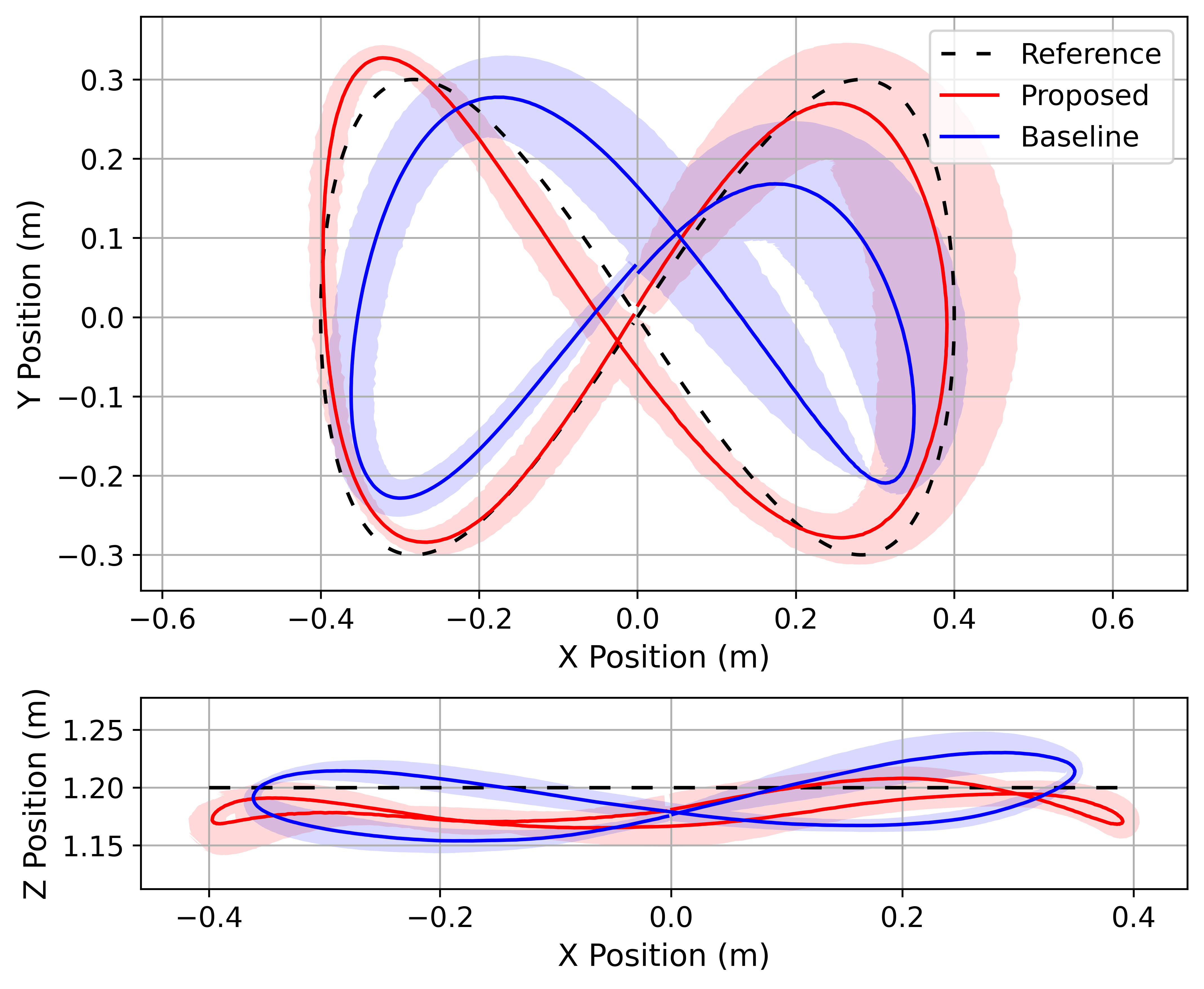}
    \caption{Experiment~1 results. Three trials were conducted along a lemniscate (figure-eight) trajectory at different average speeds. The mean and variance are shown, excluding the failed baseline case at the highest speed where only the proposed method succeeded. Overall, the proposed method tracks the desired trajectory more accurately.}
    \label{fig:exp 1}
\end{figure}

After the vehicle reached the target altitude and stabilized (5~s settling), we recorded data for five full cycles. In all three experiments with different average speeds, the proposed method successfully completed the tasks without losing stability, while the baseline diverged and lost stability during the final and fastest experiment at 1.2~m/s. The RMSE (Root Mean Square Error) of position and orientation errors for each experiment is summarized in Table~\ref{tab:table_lemni}, which shows that the proposed method consistently achieved lower RMSE across all scenarios. In Fig.~\ref{fig:exp 1}, we visualize the results of the first two experiments, excluding the failed baseline case from the third experiment, using the mean values along with $\pm1\sigma$ intervals. The variance is represented as translucent bands, while the mean values are plotted as solid lines.
% For visualization, within each condition we plot the mean position across the five cycles as a solid line with the $\pm1\sigma$ band shaded in Fig.~\ref{fig:exp 1}; the $1.2\,\mathrm{m/s}$ baseline case is omitted due to failure. 
Additionally, at $200\,\mathrm{Hz}$ we computed the position and rotation errors defined in Eq.~(8) at every time step and report the component-wise RMSE averaged over all components in Table~\ref{tab:table_lemni}.

% From the visualized means and variances, the proposed controller shows improved tracking in $x$, $y$, and $z$. Quantitatively, the RMSE metrics confirm that the proposed method outperforms the baseline across all speeds and faithfully tracks the more aggressive $1.2\,\mathrm{m/s}$ trajectory over five loops, whereas the baseline fails in this scenario.

\subsection{Experiment 2}

We further evaluate agile orientation-tracking performance by fixing the position and comparing the response under rapid roll oscillations. The desired roll angle was given as $\phi_d(t)=50^\circ \sin(2\pi f t)$, and control performance was measured while varying the oscillation frequency $f \in \{0.4, 0.6, 0.8\}\,\mathrm{Hz}$.

In Fig.~\ref{fig:exp 2}, along with the roll angle graphs, we also plot the lateral position component $y$ that shows the largest deviation. The proposed controller yields noticeably smaller $y$-axis position errors than the baseline. As in Experiment~1, the position and rotational errors for each scenario are summarized in Table~\ref{tab:table_roll}. The proposed controller shows clearer superiority at higher frequencies, quantitatively supporting the expectation that it is well suited to agile reference trajectories that demand rapid rotor and servo state changes.

\begin{table}
\centering
\caption{Experiment~2 results: Position and Orientation Error.\\Better performance in bold.} 
\label{tab:table_roll} 
\setlength{\tabcolsep}{10pt} 
\renewcommand{\arraystretch}{1.5}
\begin{tabular}{|c|l||c|c|c|}
\hline
\multicolumn{2}{|c||}{} & \multicolumn{3}{c|}{Oscillation Frequency} \\ \cline{3-5} 
\multicolumn{2}{|c||}{} & 0.4 Hz & 0.6 Hz & 0.8 Hz \\ \hline \hline
\multirow{2}{*}{Proposed} & pos. [m]   & \textbf{0.042}  & \textbf{0.037}  & \textbf{0.050}  \\ \cline{2-5} 
                          & rot. [rad] & \textbf{0.092}  & \textbf{0.098}  & \textbf{0.161}  \\ \hline \hline
\multirow{2}{*}{Baseline} & pos. [m]   & 0.050  & 0.062  & 0.095   \\ \cline{2-5} 
                          & rot. [rad] & 0.113  & 0.133  & 0.284     \\ \hline
\end{tabular}
\end{table}

\begin{figure}
    \centering
    % \includesvg[width=1.0\linewidth, inkscapelatex=false]{figures/Figure4}
    \includegraphics[width=1.0\linewidth]{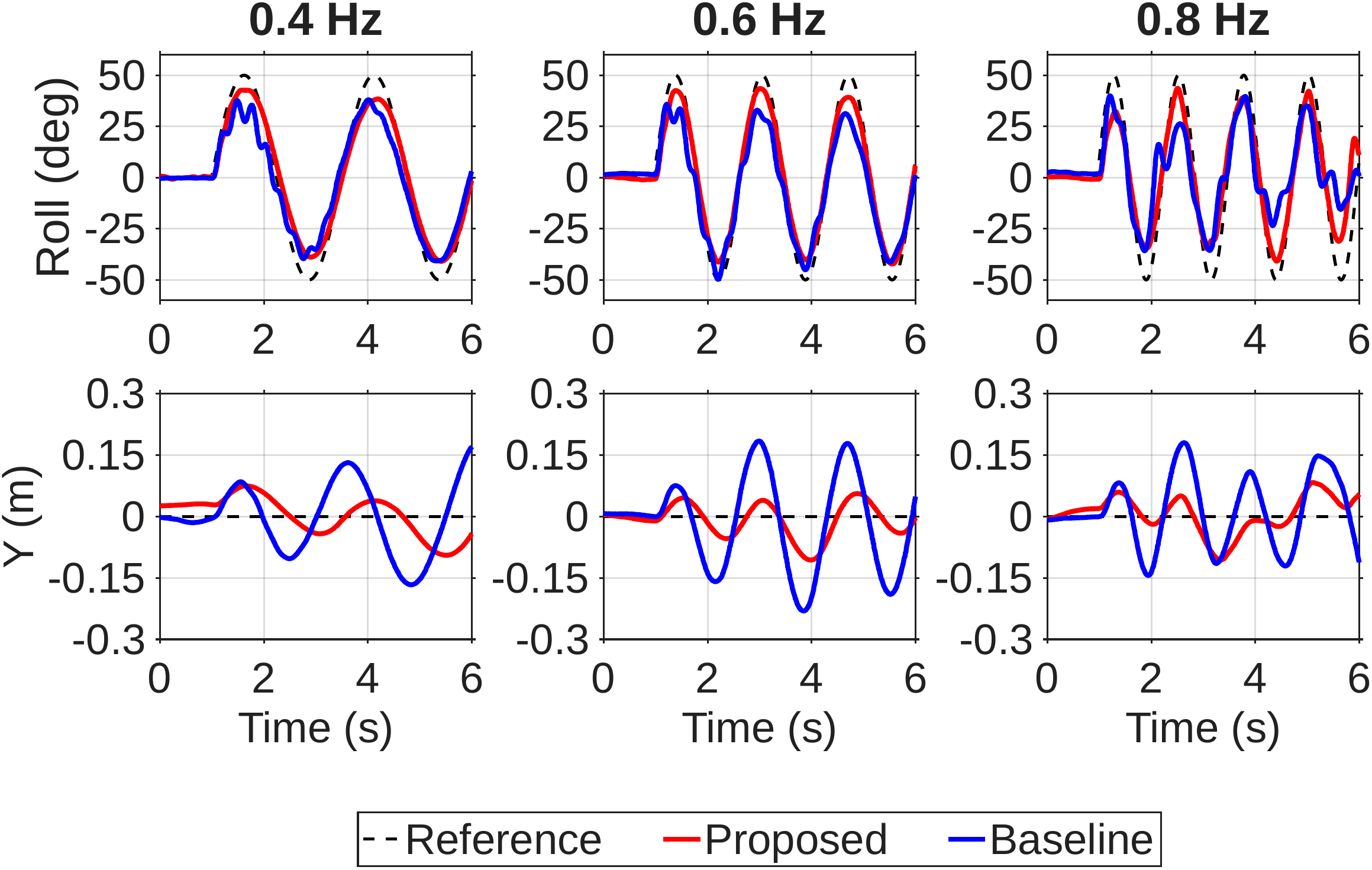}
    \caption{Experiment~2 results. With position held fixed, the roll angle was rapidly varied. The proposed controller (red) achieved more accurate position regulation and orientation tracking than the baseline (blue).}
    \label{fig:exp 2}
\end{figure}

\subsection{Experiment 3}

Finally, to compare the baseline and proposed controllers under sudden changes in the desired inputs, we generated a repeatable, large impulsive external wrench. We connect the drone to a tabletop object by a lightweight tether with initial slack. The object's mass was 0.21 kg, approximately 10\% of the drone's mass. As the drone translated away from the table, the object eventually dropped off the edge and the tether became taut, applying a short, high-magnitude disturbance to the drone (see Fig.~\ref{fig:thumbnail}). We evaluated each controller’s capability to recover position and attitude.

For a fair comparison, all trials were conducted in the same environment with an identical reference trajectory. In each flight, the drone took off and then traveled \(1\,\,\mathrm{m}\) along the \(x\)-axis at a constant speed. The recorded position and attitude are shown in Fig.~5. When the object falls, the roll angle suddenly reaches approximately $-60^\circ$. The proposed controller stabilizes the vehicle and tracks the setpoint again, while the baseline becomes unstable and diverges.

\begin{figure}
    \centering
    \includegraphics[width=1.0\linewidth]{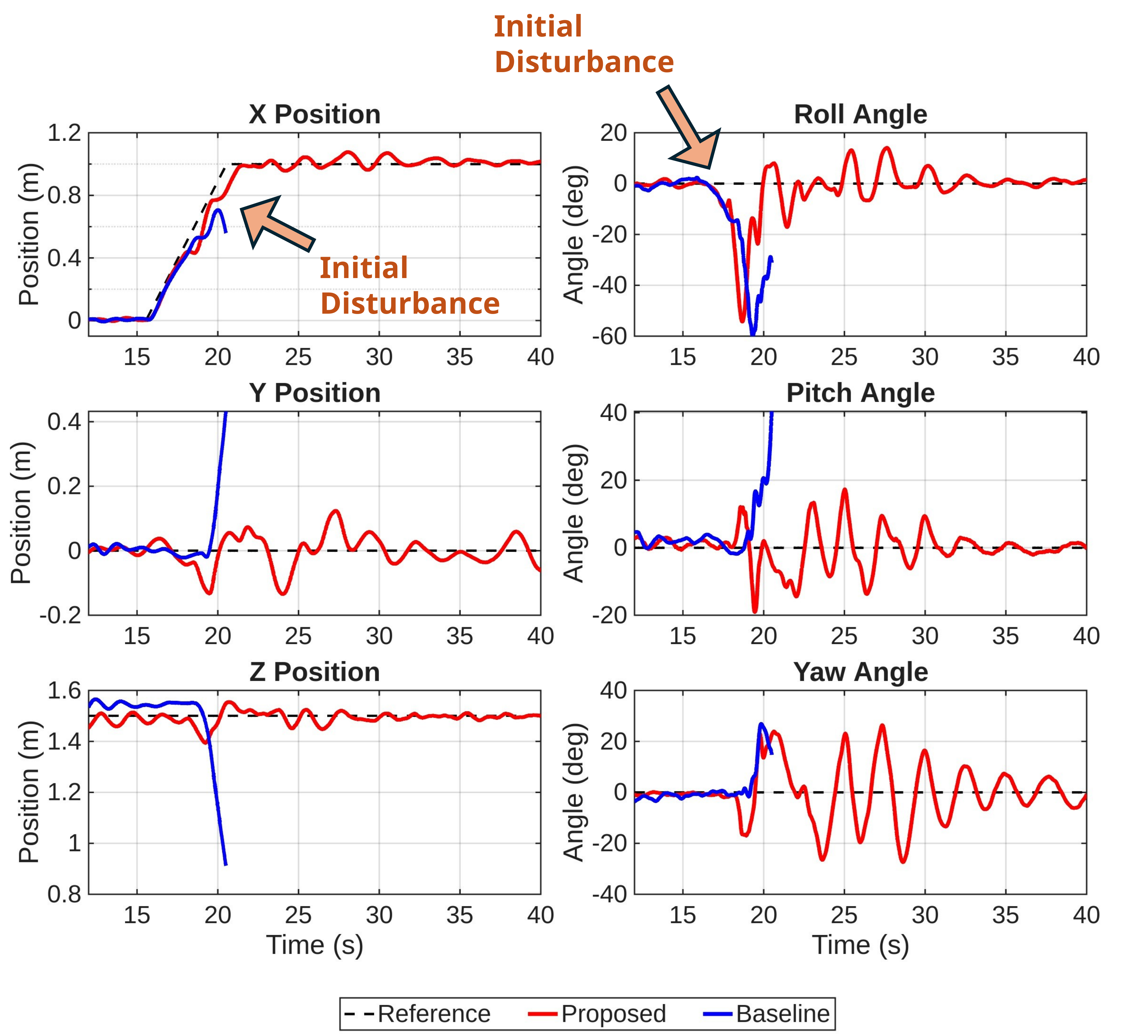}
    % \includesvg[width=1.0\linewidth, inkscapelatex=false]{figures/Figure5}
    \caption{Experiment~3 results. The initial disturbance, indicated by the yellow arrow, was caused when a 0.21 kg mass suspended by a string and initially resting on a table was pulled, applying a rotational disturbance to the multirotor. While the baseline diverged, the proposed controller stabilized the system and demonstrated superior performance.}
    \label{fig:exp 3}
\end{figure}

\section{CONCLUSION}
In this paper, we propose a geometric backstepping controller for an omnidirectional tiltrotor platform that integrates rotor and tilt-servo dynamics. Based on geometric PID control, we designed the controller within a backstepping framework and proved exponential stability under known actuator dynamics. Moreover, we confirmed through direct measurements that no single constant parameter can adequately represent the actuator time constants. To address this, we proved that the proposed controller ensures the boundedness of the closed-loop system even under bounded uncertainties of the actuator time constants. To validate the proposed controller on hardware, we conducted three experimental scenarios and compared it against a baseline geometric-allocation method. The controller demonstrated improved performance under rapid changes in control inputs. In particular, while the baseline often diverged and crashed, our method maintained stability in all cases. In future work, our aim is to strengthen the analysis to guarantee exponential stability under uncertainty and to evaluate the proposed controller across a wider range of tasks.

\normalem
\bibliographystyle{IEEEtran}
\bibliography{bibfile}

\end{document}